%% file: ms.tex
\pdfoutput=1
\documentclass{article}

\input{setup}

\title{\textbf{Meta-Learning Mean Functions \\ for Gaussian Processes}}

\author[1]{Vincent Fortuin}
\author[2]{Heiko Strathmann}
\author[1]{Gunnar R\"atsch}

\affil[1]{Institute of Machine Learning, ETH Z\"urich, Z\"urich, Switzerland}
\affil[2]{Deepmind, London, United Kingdom}

\date{}

\begin{document}

\maketitle

\input{abstract}

\input{introduction}

\input{methods}

\input{experiments}

\input{related_work}

\input{discussion}

\bibliography{library}

\input{supplement.tex}

\end{document}

%% file: setup.tex
\usepackage[utf8]{inputenc} % allow utf-8 input
\usepackage[T1]{fontenc}    % use 8-bit T1 fonts
\usepackage{booktabs}       % professional-quality tables
\usepackage{amsfonts}       % blackboard math symbols
\usepackage{amsmath}
\usepackage{dsfont}
\usepackage{nicefrac}       % compact symbols for 1/2, etc.
\usepackage{microtype}      % microtypography
\usepackage{siunitx}
\usepackage{color}
\usepackage{graphicx}
\usepackage{tabularx}
\usepackage{caption}
\usepackage{subcaption}
\usepackage[english]{babel}
\usepackage{placeins}
\usepackage{amsthm}
\usepackage{algorithm}
\usepackage{algorithmic}

\usepackage{authblk}

\usepackage[letterpaper, total={6in, 9in}]{geometry}

\usepackage[square,sort,comma,numbers]{natbib}

\newtheorem{assumption}{Assumption}
\newtheorem{proposition}{Proposition}

\DeclareMathOperator*{\argmin}{arg\,min}

\graphicspath{{./figures/}}

\newcommand{\beginsupplement}{%
        \setcounter{table}{0}
        \renewcommand{\thetable}{S\arabic{table}}%
        \setcounter{figure}{0}
        \renewcommand{\thefigure}{S\arabic{figure}}%
     }

\newcommand{\bx}{\mathbf{x}}
\newcommand{\bxstar}{\mathbf{x^*}}
\newcommand{\bystar}{\mathbf{y^*}}
\newcommand{\Kxx}{\mathbf{K_{xx}}}
\newcommand{\Kxstar}{\mathbf{K_{x*}}}
\newcommand{\Kstarstar}{\mathbf{K_{**}}}
\newcommand{\bmustar}{\mathbf{\mu^*}}
\newcommand{\by}{\mathbf{y}}
\newcommand{\bI}{\mathbf{I}}
\newcommand{\Bx}{\mathbf{B_x}}
\newcommand{\Bstar}{\mathbf{B_*}}
\newcommand{\gauss}[1]{\mathcal{N}\left( #1 \right)}
\newcommand{\given}{\, \vert \,}
\newcommand{\var}{\sigma^2}

%% file: abstract.tex
\begin{abstract}

When fitting Bayesian machine learning models on scarce data, the main challenge is to obtain suitable prior knowledge and encode it into the model.
Recent advances in meta-learning offer powerful methods for extracting such prior knowledge from data acquired in related tasks.
When it comes to meta-learning in Gaussian process models, approaches in this setting have mostly focused on learning the kernel function of the prior, but not on learning its mean function.
In this work, we explore meta-learning the mean function of a Gaussian process prior.
We present analytical and empirical evidence that mean function learning can be useful in the meta-learning setting, discuss the risk of overfitting, and draw connections to other meta-learning approaches, such as model agnostic meta-learning and functional PCA.

\end{abstract}

%% file: introduction.tex
\section{Introduction}

Bayesian methods are well suited for learning tasks with scarce data, because they offer a principled way to include prior knowledge about the problem \citep{McNeish2016using}.
A model that is particularly popular in this space due to its flexibility and data-efficiency is the Gaussian process (GP) \citep{Rasmussen2006-zv}.
If the prior knowledge encoded into the GP prior is suitable for the task, only very few observations may be needed to achieve high predictive performance.
However, the acquisition and effective encoding of this prior knowledge has been a longstanding challenge in the field \citep{Carlin2008bayesian}.

A powerful framework to acquire prior knowledge about tasks is meta-learning \citep{Vilalta2002-rj}.
Meta-learning refers to a method in which knowledge is gained by solving a set of specific tasks (the meta-tasks) and subsequently used to improve the model's performance on a different task (the target task) \citep{Vilalta2002-rj}.
The method is therefore concerned with making use of a set of tasks in order to approach another task better.
The incorporation of this prior knowledge (so called \emph{meta-knowledge}) into the learning model is called \emph{inductive transfer} \citep{Thrun2012learning}.

The GP lends itself particularly easily to the desired purpose of a target model, because it offers a way to perform nonlinear regression on small amounts of data, while being able to incorporate prior knowledge in a Bayesian manner \citep{Rasmussen2006-zv}.
This prior knowledge can be encoded into the GP model by modifying the parameters of either its kernel function or its mean function.
Previous work on meta-learning in Gaussian processes has mostly focused on the kernel function \citep{Bonilla2008-rj, Widmer2010inferring, Skolidis2012-di}, while the mean function has been largely ignored.

In this work, we explore using the mean function for meta-learning in GPs.
We discuss potential reasons why this has not previously been done and why these reasons might not be well-founded.
Moreover, we present analytical arguments for mean function learning and validate them empirically on different tasks, as well as drawing connections to other successful meta-learning approaches. 

We make the following contributions:

\begin{itemize}
	\item Analyse the risk of overfitting for mean function learning in Gaussian processes in standard supervised learning as well as meta-learning.
	\item Present analytical arguments and empirical evidence that mean function learning can be beneficial in the meta-learning setting.
	\item Discuss connections to model-agnostic meta-learning and functional PCA.
\end{itemize}

In the following, we are going to formalize the idea of GP prior learning in a meta-learning setting (Sec.~\ref{sec:meta-learning}), discuss the potential risk of overfitting when learning the parameters of a GP prior (Sec.~\ref{sec:overfitting}), present an argument for why mean function learning can be useful under certain conditions (Sec.~\ref{sec:superior_mean}), and discuss the connections between the proposed approach and model-agnostic meta-learning (Sec.~\ref{sec:maml}) as well as functional PCA (Sec.~\ref{sec:fpca}).
Thereafter, we are going to provide empirical evidence of our claims (Sec.~\ref{sec:experiments}), a review of the related literature (Sec.~\ref{sec:related_work}), and a discussion of our work (Sec.~\ref{sec:conclusion}).
An overview of our proposed method is depicted in Figure~\ref{fig:overview}.
Inspired by the idea of deep kernel learning \citep{Wilson2015-ji}, we parameterize all our meta-learnable kernels and mean functions in this paper as deep neural networks (more details in the appendix in Sec.~\ref{sec:deep_mean}).

%% file: methods.tex
\section{Meta-learning in Gaussian Processes}
\label{sec:meta-learning}

\begin{figure*}
\centering
\includegraphics[width=0.9\linewidth]{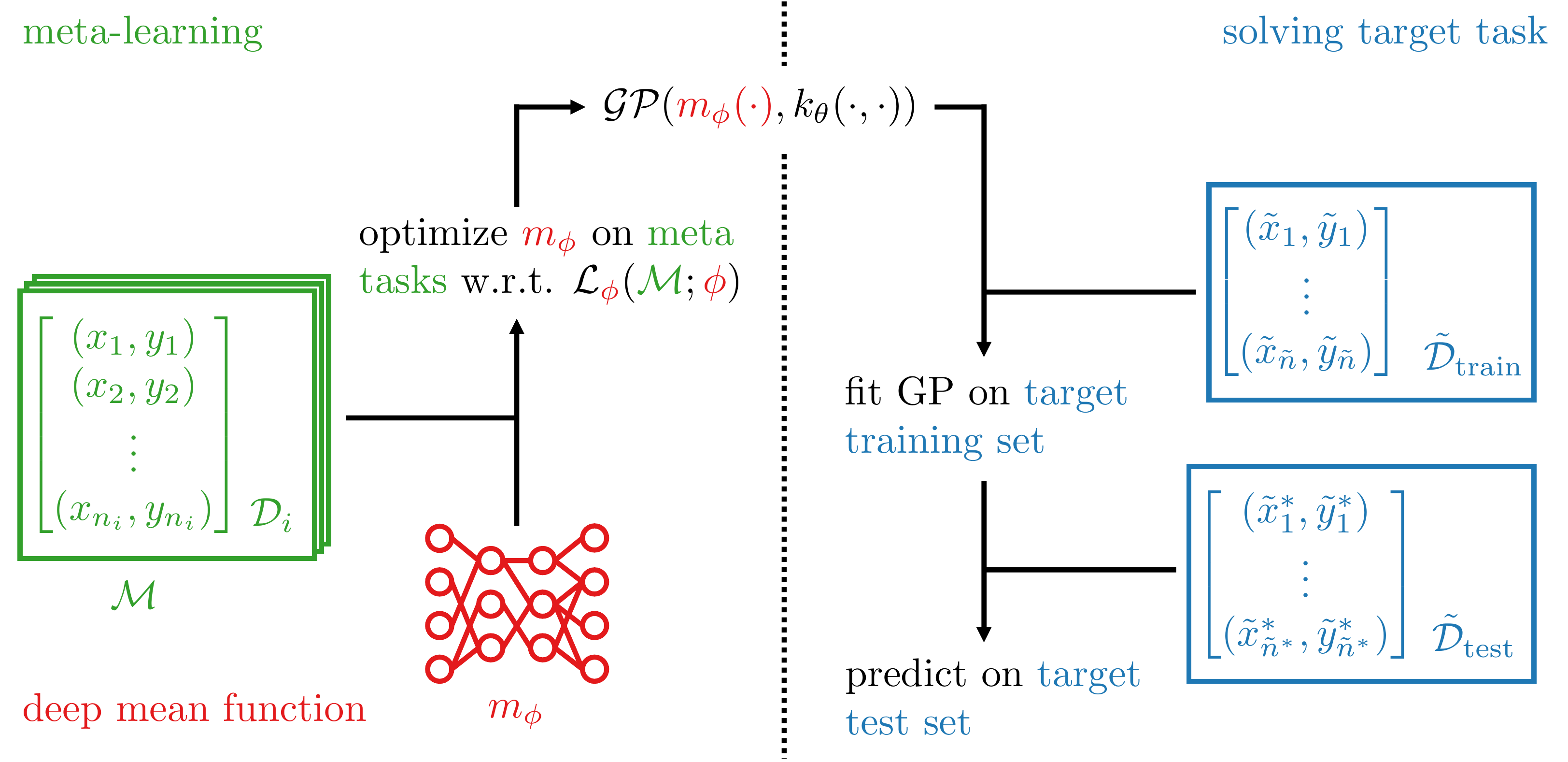}
\caption{Overview of the proposed deep mean function meta-learning framework. A similar procedure is used to meta-learn the kernel function parameters $\theta$, or both, mean and kernel function parameters.
}
\label{fig:overview}
\end{figure*}

In order to explain the setting of this work, we are first going to define meta-learning more formally.
We are following the common setting from \citet{baxter2000model}, where we have a so-called \emph{environment} $\tau$ and our tasks are sampled \emph{i.i.d.} from this environment.
For each task, we sample a data distribution and a sample size from the environment $(\mathcal{P}_i, n_i) \sim \tau$ and then sample a data set $\mathcal{D}_i \sim \mathcal{P}_i^{n_i}$, where $\mathcal{P}_i^{n_i}$ is the distribution of data sets of size $n_i$ under $\mathcal{P}_i$.

The set of meta-tasks $\mathcal{M}$ consists of $m$ data sets $\mathcal{M} := \lbrace \mathcal{D}_i \rbrace_{i=1}^{m}$, with one data set for each meta-task.
Each of these data sets contains observations $\mathcal{D}_i := \lbrace \mathbf{x}_i, \mathbf{y}_i \rbrace$, where $\mathbf{x}_i \in \mathbb{R}^{n_i \times d}$ and $\mathbf{y}_i \in \mathbb{R}^{n_i \times p}$ for tasks in which the respective functions to be learned are defined as $f_i : \mathbb{R}^d \rightarrow \mathbb{R}^p$.
In this setting, all meta-tasks share the same input and output dimensionalities $d$ and $p$, but they can have different numbers of observations $n_i$.

Additionally to the meta-tasks, we have a target task with a data set $\mathcal{\tilde{D}} := \lbrace \mathbf{\tilde{x}}, \mathbf{\tilde{y}}, \mathbf{\tilde{x}^*}, \mathbf{\tilde{y}^*} \rbrace$, where $\mathbf{\tilde{x}} \in \mathbb{R}^{\tilde{n} \times d}$ and $\mathbf{\tilde{y}} \in \mathbb{R}^{\tilde{n} \times p}$ are the training points and their respective values and $\mathbf{\tilde{x}^*} \in \mathbb{R}^{\tilde{n}^* \times d}$ and $\mathbf{\tilde{y}^*} \in \mathbb{R}^{\tilde{n}^* \times p}$ are the test points and their respective values.
We assume this target task to also be sampled from the environment $\tau$.
As mentioned above, we assume there to be much more data in the meta-tasks than in the target task, that is, $n_{\text{train}} = \sum_{i = 1}^m n_i \gg \tilde{n}$.

In order to predict on $\lbrace \mathbf{\tilde{x}^*, \tilde{y}^*} \rbrace =: \mathcal{\tilde{D}}_{\text{test}}$, we want to fit a GP to $\lbrace \mathbf{\tilde{x}}, \mathbf{\tilde{y}} \rbrace =: \mathcal{\tilde{D}}_{\text{train}}$ with a prior
\begin{equation}
f(\cdot) \sim \mathcal{GP}(m_\phi (\cdot), k_\theta(\cdot,\cdot)) \; ,
\label{eq:gp_prior}
\end{equation}
where the mean and kernel functions are parameterized by sets of parameters $\phi$ and $\theta$ respectively.
These parameters can now be optimized on the meta-task set, that is,
$\phi^* = \argmin_\phi \mathcal{L}_{\phi} \left( \mathcal{M}; \phi \right)$ and
$\theta^* = \argmin_\theta \mathcal{L}_{\theta} \left( \mathcal{M}; \theta \right)$
with a suitable loss function $\mathcal{L}_\psi$ for parameters $\psi \subset \{ \phi, \theta \}$.
This approach can be seen as gaining knowledge from solving the meta-tasks $\mathcal{M}$ and using the parameters $\phi$ and $\theta$ to encode the thus acquired meta-knowledge into the GP prior for the task on $\mathcal{\tilde{D}}$.

In GP regression, the loss function is often chosen to be the \emph{negative log marginal likelihood} (LML) \citep{Rasmussen2006-zv}.
The LML on a meta-task can be computed as
\begin{align}
	\mathcal{L}_{\psi} (\mathcal{D}_i ; \psi) &= - \log p ( \mathbf{y}_i \vert \mathbf{x}_i ; \psi) \nonumber \\
	&= - \log \int p(\mathbf{y}_i \vert \mathbf{f}_i, \mathbf{x}_i) \, p(\mathbf{f}_i \vert \mathbf{x}_i ; \psi) \, d\mathbf{f}_i \; ,
	\label{eq:LML}
\end{align}
where all learnable parameters of the GP prior (i.e., the mean parameters, kernel parameters, or both) are denoted as $\psi$ and $\mathbf{f}_i$ is a vector of the function values of the latent function $f_i(\cdot)$ evaluated at the points $\mathbf{x}_i$.

Given this loss function over a single meta-task, we define the loss over all meta-tasks as a sum over their individual losses, that is,
\begin{equation}
	\mathcal{L}_{\psi} \left( \mathcal{M}; \psi \right) := \sum_{i=1}^m \mathcal{L}_{\psi} \left( \mathcal{D}_i; \psi \right) \; .
\end{equation}

This loss can then be optimized using any general-purpose optimization method.
In this work, we use Stochastic Gradient Descent (SGD) (Alg.~\ref{alg:training}).

\begin{algorithm}
\caption{Algorithm to optimize the GP prior parameters $\psi$ on the set of meta-tasks $\mathcal{M}$ using SGD.}
\begin{algorithmic}
	\REQUIRE	 set of meta-tasks $\mathcal{M} = \lbrace \mathcal{D}_i \rbrace_{i=1}^{m}$, learning rate $\eta$
	\WHILE{not converged}
		\FORALL{$\mathcal{D}_i \in \mathcal{M}$}
			\STATE {Compute $\mathcal{L}_{\psi} (\mathcal{D}_i ; \psi) = - \log p ( \mathbf{y}_i \vert \mathbf{x}_i ; \psi)$}
			% \COMMENT{see Eq.~\eqref{eq:LML}}
			\STATE {Update $\psi \leftarrow \psi - \eta \nabla_{\psi} \mathcal{L}_{\psi} (\mathcal{D}_i ; \psi)$}
		\ENDFOR
	\ENDWHILE
\end{algorithmic}
\label{alg:training}	
\end{algorithm}

Once the parameters of the GP prior are optimized, we can use the prior to fit a GP to $\mathcal{\tilde{D}}_{\text{train}}$ and predict on $\mathcal{\tilde{D}}_{\text{test}}$.
If we evaluate the predictive posterior of the GP on the test points, it yields \citep{Rasmussen2006-zv}
\begin{align}
\label{eq:gp-posterior}
&\qquad \quad p \left( \mathbf{\tilde{f}^*} | \mathbf{\tilde{y}}, \mathbf{\tilde{x}}, \mathbf{\tilde{x}^*} \right) = \mathcal{N} \left( \mathbf{m^*}, \mathbf{K^*} \right) \\
&\text{with} \nonumber \\
&\mathbf{m^*} = m_\phi(\mathbf{\tilde{x}^*}) + K_\theta^{*x} \left( K_\theta^{xx} + \sigma^2 I \right)^{-1} (\mathbf{\tilde{y}} - m_\phi(\mathbf{\tilde{x}})) \nonumber \\
&\mathbf{K^*} = K_\theta^{**} - K_\theta^{*x} \left( K_\theta^{xx} + \sigma^2 I \right)^{-1} K_\theta^{x*} \nonumber
\end{align}
where $K^{x*}_\theta = K^{*x\top}_\theta$ denotes the kernel matrix (also known as the Gram matrix) with $(K^{x*}_\theta)_{ij} = k_\theta(\tilde{x}_i,\tilde{x}_j^*)$ and similarly for $K^{xx}_\theta$ and $K^{**}_\theta$.

\section{Meta-learning mean functions for GPs}
\label{sec:model}

In this section, we are going to lay out the main theoretical arguments for our approach.
Firstly, we are going to give an intuition for why overfitting can be a risk in conventional mean function learning, but not in a meta-learning setting.
Secondly, we are going to present an analytical treatment of mean function and kernel function learning in a meta-learning setting and show why mean function learning can be useful under certain conditions.
Lastly, we are going to highlight connections between our proposed approach and popular meta-learning methods, such as model-agnostic meta-learning and functional PCA.

\subsection{On the risk of overfitting in GP prior learning}
\label{sec:overfitting}

Optimizing hyperparameters in machine learning models always brings about a certain risk of overfitting.
The extent of this risk depends on the data set on which different hyperparameters are optimized, on the objective function, and on the procedure that is used to optimize this function.

One of the most principled ways of choosing the hyperparameters $\psi$ of the GP prior is Bayesian model selection \citep{Rasmussen2006-zv}.
Bayesian model selection chooses the hyperparameters according to their posterior probabilities given the data.
These probabilities can be computed as
\begin{equation}
	p ( \psi \vert \mathbf{y}, \mathbf{x} ) = \frac{p(\mathbf{y} \vert \mathbf{x}, \psi) \, p(\psi)}{p(\mathbf{y} \vert \mathbf{x})} \; ,
	\label{eq:hyper-posterior}
\end{equation}
where the denominator is given by
\begin{equation}
	p(\mathbf{y} \vert \mathbf{x}) = \int p(\mathbf{y} \vert \mathbf{x}, \psi) \, p(\psi) \, d\psi \, .
	\label{eq:hyper-integral}
\end{equation}

Instead of this maximum \emph{a posteriori} (MAP) inference over $\psi$, practitioners often resort to a maximum likelihood estimate (MLE) for reasons of tractability, by optimizing the likelihood term in the numerator of \eqref{eq:hyper-posterior} with respect to $\psi$.
Note that the negative logarithm of this term is exactly the LML from \eqref{eq:LML}.
Since this approach uses an MLE method to optimize the \emph{hyperparameters} instead of the \emph{parameters} of the model, it is sometimes called \emph{type II} maximum likelihood (ML-II) approximation \citep{Rasmussen2006-zv}.

For the GP model from \eqref{eq:gp-posterior}, the LML can be computed in closed form \citep{Rasmussen2006-zv} as
\begin{equation}
\begin{split}
	\log &p ( \mathbf{\tilde{y}} \vert \mathbf{\tilde{x}}, \psi) = \\
	&\; - \frac{1}{2} (\mathbf{\tilde{y}} - m_\phi(\mathbf{\tilde{x}}))^{\top} \left( K_\theta^{xx} + \sigma^2 I \right)^{-1} (\mathbf{\tilde{y}} - m_\phi(\mathbf{\tilde{x}})) \\
	&\; - \frac{1}{2} \log \vert K_\theta^{xx} + \sigma^2 I \vert \\
	&\; - \frac{n}{2} \log 2 \pi \; ,
\end{split}
\label{eq:explicit-LML}
\end{equation}
with hyperparameters $\psi := (\phi, \theta)$ and where $\vert M \vert$ denotes the determinant of matrix $M$.

It is evident that this LML naturally decomposes into three terms.
The first term depends on the kernel parameters, the mean function parameters, and the data.
It can be seen as measuring the goodness-of-fit of the model to the data.
The second term only depends on the kernel parameters and can be seen as a complexity penalty.
The third term normalizes the likelihood and is constant with respect to the data and the parameters.

Notice that this objective function contains an automatic tradeoff between data-fit and model complexity when it comes to the kernel parameters, but it does not contain any complexity penalty on the mean function parameters.
Optimizing the mean function parameters w.r.t.\ the LML directly on the training data can therefore lead to overfitting (as we will also see empirically in Sec.~\ref{sec:experiments}).
This might be one of the reasons why mean function learning has not been very popular in the field so far.
Moreover, it has been noted that even for the kernel parameters, the penalty term might sometimes not be strong enough to consistently avoid overfitting \citep{Rasmussen2006-zv}, even though this effect can be additionally combated with regularization \citep{Micchelli2005learning, Cawley2007preventing}.

In the meta-learning setting however, we refrain entirely from optimizing any parameters using the LML on the target training data $\mathcal{\tilde{D}}_{\text{train}}$.
Instead, we only optimize the GP prior's parameters on the meta-tasks $\mathcal{M}$.
Thus, there is no way for the training data of the target task to inform the GP prior and therefore no possibility of overfitting to the training data.
Due to the reduced risk of overfitting, mean function learning seems applicable in this setting, while it is rightfully avoided in the standard supervised learning setting.

\subsection{Mean functions can be useful for meta-learning}
\label{sec:superior_mean}

While learning kernel functions for GPs is a popular area of research, it entails a number of challenges.
The most fundamental challenge is that the kernel function has to be positive definite in order to define an inner product in a suitable Hilbert space \citep{Mercer1909-dq}.
GP mean functions, in contrast, do not suffer from any such constraints and can therefore be learned more freely without taking any special precautions.

The question remains whether we can incorporate useful prior knowledge into the mean function or whether the kernel function is sufficient to encode all our prior knowledge.
Given that many approaches in the literature resort to using a zero mean function and only adjust the kernel, they seem to be implicitly making the following assumption:

\begin{assumption}
	\label{lem:prior_mean}
	Given a parameterized Gaussian Process prior $\mathcal{GP}(m_\phi (\cdot), k_\theta(\cdot,\cdot))$, all the available prior knowledge can be effectively incorporated into the kernel parameters $\theta$ when a na\"ive mean function (i.e., $m_\phi (\cdot) = 0$) is used, such that complete prior knowledge yields an optimal predictive performance.
\end{assumption}

Based on \eqref{eq:gp-posterior}, it is easy to see that this assumption does not always hold.
One can for instance construct a counter-example by having a function that takes zero values in certain parts of the domain and nonzero values in other parts.
If this function is known \emph{a priori}, it can be used as a mean function for the GP prior, leading to an optimal posterior fit.
In contrast to that, there is no kernel that can guarantee an optimal fit if a zero mean function is used.
The analytical details for this counter-example are deferred to the appendix (Sec.~\ref{sec:model_details}).

Moreover, especially in the meta-learning setting, we often deal with a very small number of observations in the target task.
Since the mean function acts on the whole input domain, independent of the number of observations, we would expect its influence to be stronger than the kernel's in such tasks.

Given these insights and the fact that learning mean functions poses a less constrained problem than kernel learning, we propose that mean functions might be a better choice than kernels for meta-learning in GPs under certain conditions.
We will empirically validate this hypothesis in the experimental section (Sec.~\ref{sec:experiments}).

\begin{figure*}
	\centering
	%TODO: make this figure more readable
	\includegraphics[width=0.9\linewidth]{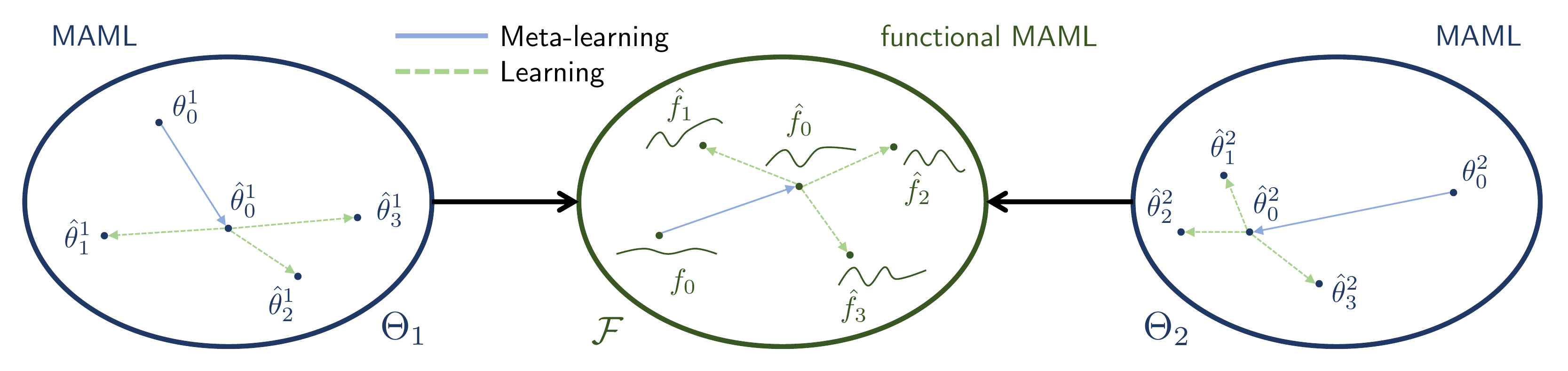}
	\caption{Sketch of the mechanism of MAML \citep{Finn2017-ef} in two different parameterizations of the learning model and of our mean function learning approach. Our approach acts directly in the function space $\mathcal{F}$ instead of the parameter spaces $\Theta_i$, such that for instance the parameters $\hat{\theta}_1^1$ and $\hat{\theta}_1^2$ get mapped to the same function $\hat{f}_1$.}
	\label{fig:maml}
\end{figure*}

\section{Connections to other meta-learning approaches}

Two popular and time-tested methods for meta-learning are model-agnostic meta-learning \citep{Finn2017-ef} and functional PCA \citep{rice2001nonparametric}.
It turns out that these approaches can be related to mean function meta-learning in GPs, thus lending further support to our proposed approach.
We discuss those connections in the following.

\subsection{Model-agnostic meta-learning}
\label{sec:maml}

Model-agnostic meta-learning (MAML) \citep{Finn2017-ef} is a popular framework for meta-learning, especially with neural networks.
The general idea of this method is to find the best initial parameters $\hat{\theta}_0 \in \Theta$ for the target model, such that the required learning updates to get from $\hat{\theta}_0$ to the optimal parameters $\hat{\theta}_i$ for each meta-task $\mathcal{D}_i$ are minimized.

An issue with this framework is the fact that the parameterization of the models is arbitrary.
One can for instance choose two different parameterizations of the model with parameter spaces $\Theta_1$ and $\Theta_2$, leading to different optimal parameters $\{ \hat{\theta}^1_i \in \Theta_1 \}$ and $\{ \hat{\theta}^2_i \in \Theta_2 \}$.
\emph{A priori}, it is not clear which one of these parameterizations (or of all the other possible parameterizations) will yield a better meta-learning performance.

Moreover, in over-parameterized models such as neural networks, there are exponentially many weight space symmetries, that is, models with different parameter vectors that encode the exact same function \citep{badrinarayanan2015understanding}.
If one would take a subset of those and try to find a parameter that minimizes the average distance to all of them, it would in all likelihood end up encoding a different function even though all of the target parameters encode the same one.

In contrast to this, our proposed mean function learning framework acts directly in the function space $\mathcal{F}$ and tries to find the best mean function $\hat{f}_0$ that requires the least number of observations to fit the respective posterior means $\hat{f}_i$ of the meta-tasks well.
It is thus independent of the parameterization of the functions and optimizes the prior directly with respect to the geometry of the associated function space, which can be argued to be more efficient than any arbitrary parameter space \citep{amari1998natural}.
Additionally, its independence on the parameterization would also allow it to compare functions with different numbers of parameters in the same space, which would not be possible with MAML.
Given this perspective, our approach can be seen as a function-space equivalent of MAML (Fig.~\ref{fig:maml}).

\subsection{Functional principal component analysis}
\label{sec:fpca}

Functional PCA (FPCA) is a method from functional data analysis where a function $y(x)$ is modeled as a linear combination of a small number of principal components in a given function space \citep{rice2001nonparametric, james2000principal}.
The function space is often chosen to be parameterized using B-splines as basis functions, such that a point $x$ is represented as a vector $\mathbf{b}(x) = \left[ b_1(x), \dots, b_q(x) \right]^\top$, where the $b_i(\cdot)$'s are $q$ mutually orthogonal splines.
A vector of points $\bx = \left[ x_1, \dots, x_n \right]$ is therefore represented as a matrix $\Bx$ with $(\Bx)_{ij} = b_j(x_i)$.

The generative assumption of FPCA is a linear mixed effect model, where the function $y_i$ is given as
\begin{equation}
	\by_i(\bx) = \Bx \beta + \Bx \gamma_i + \epsilon \; ,
\end{equation}
with $\epsilon \sim \gauss{0, \var}$ and parameters $\{ \beta, \gamma_i \}$. The parameter $\beta$ is shared between different functions, whereas the $\gamma_i$'s are function-specific.
The $\gamma_i$'s are often chosen to have zero mean and covariance $\Gamma$ \citep{rice2001nonparametric}, such that the distribution over $\by_i$ becomes
\begin{equation}
\label{eq:fpca}
	p(\by_i \given \bx; \beta, \Gamma) = \gauss{\Bx \beta, \var \bI + \Bx \Gamma \Bx^\top} \; .
\end{equation}

In reduced-rank models, the covariance can further be approximated by $\Gamma = \Theta \mathbf{D} \Theta^\top$, where $\Theta \in \mathbb{R}^{q \times m}$, $\mathbf{D} \in \mathbb{R}^{m \times m}$ and diagonal, and $m \ll q$ \citep{james2000principal}.

We make the following observation:

\begin{proposition}
\label{prop:gp_fpca}
	When choosing a GP prior with mean function $m(x) = \mathbf{b}(x)^\top \theta_m$ and kernel $k(x, x') = \mathbf{b}(x)^\top \mathbf{b}(x')$, where $\theta_m$ is a meta-learnable parameter, the predictive posterior reduces to the one of FPCA.
\end{proposition}
\begin{proof}
	The proof is deferred to the appendix (Sec.~\ref{sec:fpca_proof}).
\end{proof}

Similarly to our mean function learning approach, in the FPCA setting the mean function parameter $\theta_m$ would be optimized using maximum likelihood across all functions (which we would call meta-tasks).
This formulation of FPCA can therefore be seen as a form of meta-learning in GPs, where the mean and kernel function are restricted to be linear, they are sharing an input representation based on a set of spline basis functions, and only the mean function is meta-learned.

%% file: experiments.tex
\section{Experiments}\label{sec:experiments}

In order to assess the performance of mean function meta-learning and compare it to kernel meta-learning, we performed experiments on synthetic data, on MNIST handwritten digits \citep{LeCun1998-es}, and on a challenging real-world medical data set from the 2012 Physionet Challenge \citep{Silva2012predicting}.
We found strong qualitative and quantitative evidence that mean function learning (on its own and in combination with kernel learning) can outperform kernel learning alone in a wide range of meta-learning scenarios.
In the following, we present the main results, while experimental details are deferred to the appendix (Sec.~\ref{sec:exp_details}).

\begin{table*}
    \centering
    \caption{Performance comparison of GPs with a zero mean function and RBF kernel (\emph{vanilla}), a learned mean function and RBF kernel (\emph{learned mean}), a zero mean and learned kernel function (\emph{learned kernel}) and a learned mean and learned kernel function (\emph{learned both}) on step function regression for different numbers of training points. The performance is measured in terms of likelihood and mean squared error. The values are means and their standard errors of 1000 randomly sampled step functions and sets of training points. Learning both the mean function and the kernel function consistently outperforms the other methods.}
    \resizebox{\textwidth}{!}{
    \begin{tabular}{lrrrrrr}
\toprule
& \multicolumn{2}{c}{$\tilde{n} = 1$} & \multicolumn{2}{c}{$\tilde{n} = 5$} & \multicolumn{2}{c}{$\tilde{n} = 20$} \\
        Method & \multicolumn{1}{c}{likelihood} & \multicolumn{1}{c}{MSE} & \multicolumn{1}{c}{likelihood} & \multicolumn{1}{c}{MSE} & \multicolumn{1}{c}{likelihood} & \multicolumn{1}{c}{MSE} \\
\midrule
vanilla \citep{Rasmussen2006-zv}      &  0.28 $\pm$ 0.00 &  0.335 $\pm$ 0.006 &  0.35 $\pm$ 0.00 &  0.113 $\pm$ 0.003 &  0.47 $\pm$ 0.01 &  0.026 $\pm$ 0.001 \\
learned kernel \citep{Wilson2015-ji} &  0.54 $\pm$ 0.00 &  0.323 $\pm$ 0.007 &  \textbf{0.61 $\pm$ 0.00} &  0.093 $\pm$ 0.004 &  0.68 $\pm$ 0.01 &  \textbf{0.024 $\pm$ 0.001} \\
learned mean (ours) &  0.29 $\pm$ 0.00 &  0.196 $\pm$ 0.002 &  0.36 $\pm$ 0.00 &  0.095 $\pm$ 0.002 &  0.46 $\pm$ 0.01 &  0.027 $\pm$ 0.001 \\
learned both (ours)        &  \textbf{0.55 $\pm$ 0.00} &  \textbf{0.186 $\pm$ 0.002} &  \textbf{0.61 $\pm$ 0.00} &  \textbf{0.078 $\pm$ 0.002} &  \textbf{0.69 $\pm$ 0.01} &  \textbf{0.022 $\pm$ 0.001} \\
\bottomrule
\end{tabular}
}
\label{tab:step_functions}
\end{table*}

\paragraph{Baselines}

Throughout the experiments, we mainly compare four different approaches: (a) a na\"ive GP prior (zero mean and RBF kernel) which we call \emph{vanilla} GP; (b) deep kernel learning \citep{Wilson2015-ji} with a zero mean function; (c) a learned deep mean function (our proposed model) with an RBF kernel; and (d) a learned deep mean function with a learned deep kernel.
The neural networks used for the learned mean functions and the learned kernels share the exact same architecture, except for the last layer (which is of size 1 for the mean function and of size 2 for the kernel).
Note that we aim to purely compare the benefits of learning kernels and mean functions in general here and not to achieve the maximum performance on the chosen tasks.
We hence refrain from comparing these approaches to more sophisticated meta-learning models and leave this to future work.

\paragraph{Performance measures.}
In our experiments, we assess the performance of the models on the target tasks with two different measures: the test mean squared error (MSE) and the test data log likelihood (often just denoted as \emph{likelihood}).
Note that the likelihood depends on the whole predictive posterior, while the MSE only depends on its mean.
Since the mean function of the GP only affects the posterior by shifting its mean, it can be hypothesized that learning a good mean function for the GP will affect the MSE more strongly than the likelihood.
Similarly, since the kernel function parameterizes the covariance of the process, it could be expected to affect the likelihood more strongly than the MSE.
When interpreting the results of our experiments, one should therefore keep in mind that the MSE slightly favors GPs with a good mean function, while the likelihood slightly favors good kernel functions.
\newline
However, the decision which one of the two measures is more important depends on the intended use of the GP in the target task.
If the GP is used to predict values at test points from its predictive posterior mean, the MSE is the more relevant measure.
If it is instead used to draw multiple samples from the whole posterior or to estimate the probability of different outcomes, the likelihood is more relevant.
We do not make any limiting assumptions on the use cases in this work and the ultimate decision for a measure (and therefore potentially a preferred model) is left to the practitioner.

\paragraph{Step function regression.}

\begin{figure}
\centering
\includegraphics[width=0.8\linewidth]{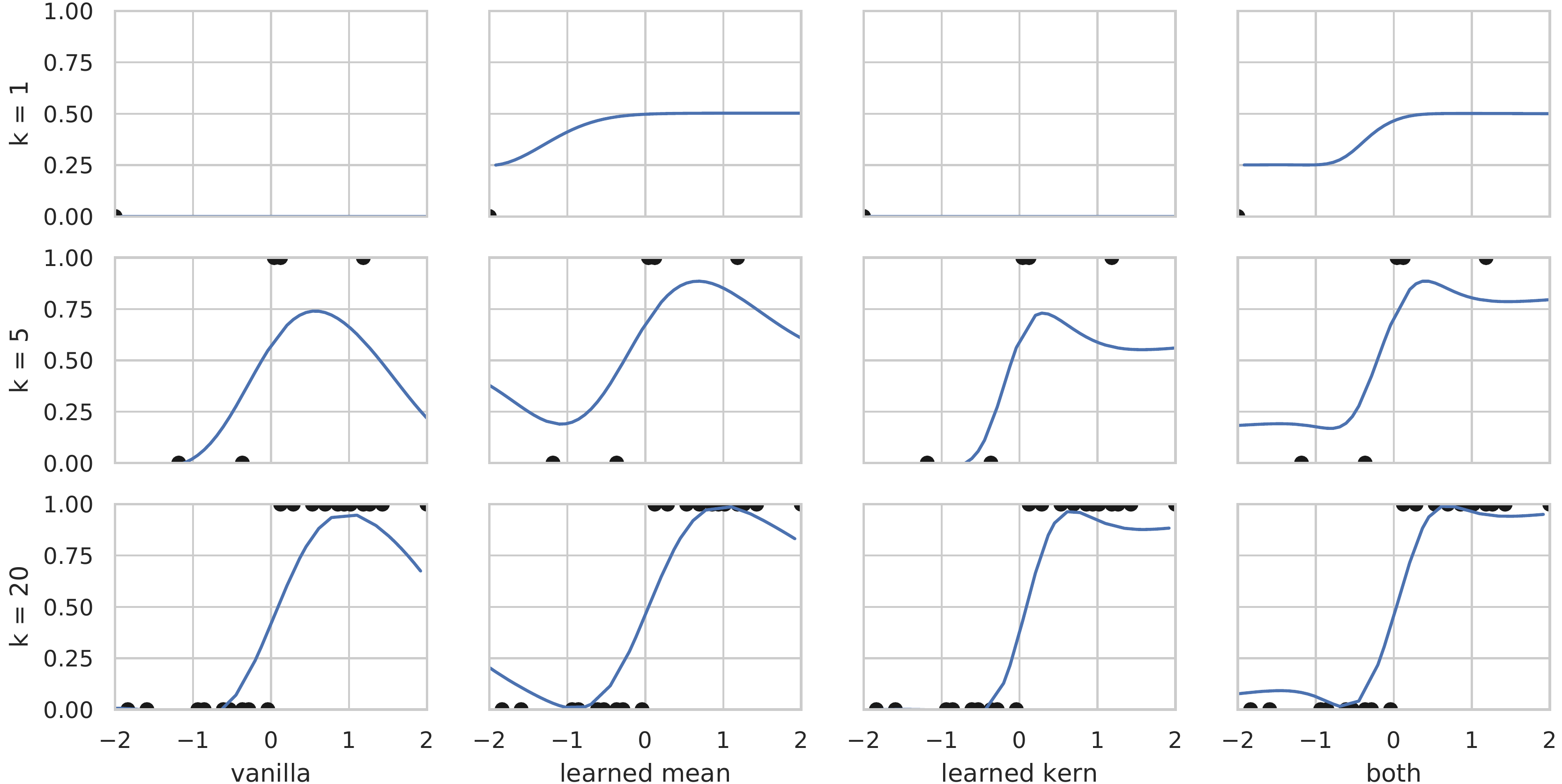}
\label{fig:step_functions}
\caption{Reconstructions of Heaviside step functions using the different methods and different numbers of training points $k$.
}
\end{figure}

To assess the performance of the mean function learning approach on a traditionally more challenging problem for GPs and compare it to kernel learning, we chose the task of step function regression.
Step functions are hard to fit for GPs due to their discontinuity \citep{Rasmussen2006-zv}.
We hypothesize that this discontinuity can be modeled more easily by a mean function than by a kernel function, since the kernel function always interpolates between neighboring points to some extent and therefore implicitly assumes continuity.
\newline
The respective performances of the methods in terms of likelihood and MSE are reported in Table~\ref{tab:step_functions} and some example regression outputs on the standard Heaviside step function for the different models and numbers of training points are depicted in Figure~\ref{fig:step_functions}.
It can be seen that in the low-data regime, mean function learning outperforms kernel learning in terms of MSE, while kernel learning yields a better likelihood.
However, the gap between mean function learning and kernel learning in terms of MSE narrows with an increasing number of training points, following our intuition that the prior becomes less important with increasing amounts of data.
\newline
It can also be seen that learning both the mean and the kernel function consistently performs best in all data regimes.
Especially in the low-data regime, we can see that this approach combines the high likelihood of the kernel learning with the low MSE of the mean function learning.
It stands to question, however, whether this effect is due to the relative simplicity of the task or whether it also holds for more complex data.

\begin{figure*}
	\centering
	\begin{subfigure}{0.45\linewidth}
		\includegraphics[width=0.95\linewidth]{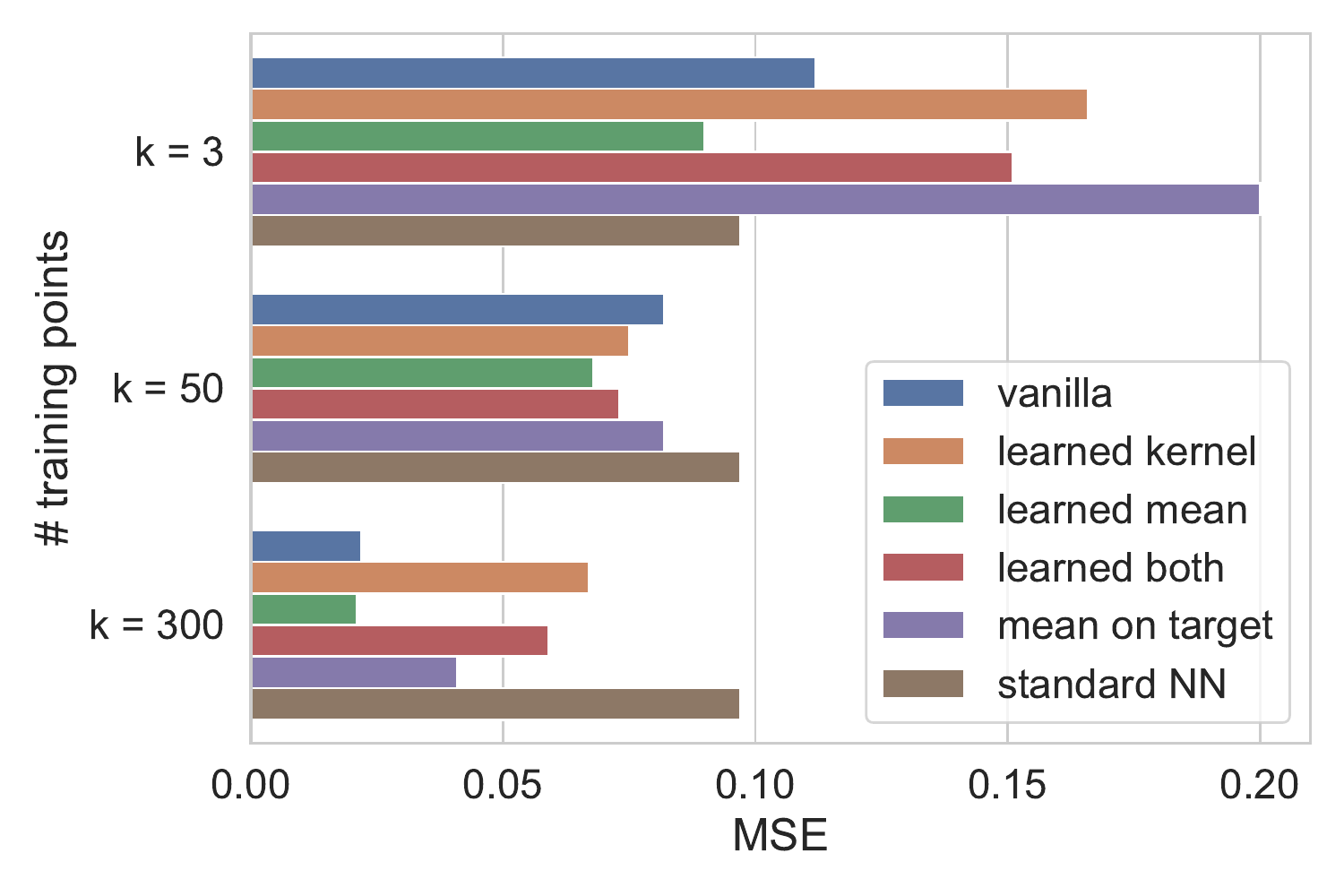}
	\end{subfigure}
	\begin{subfigure}{0.45\linewidth}
		\includegraphics[width=0.95\linewidth]{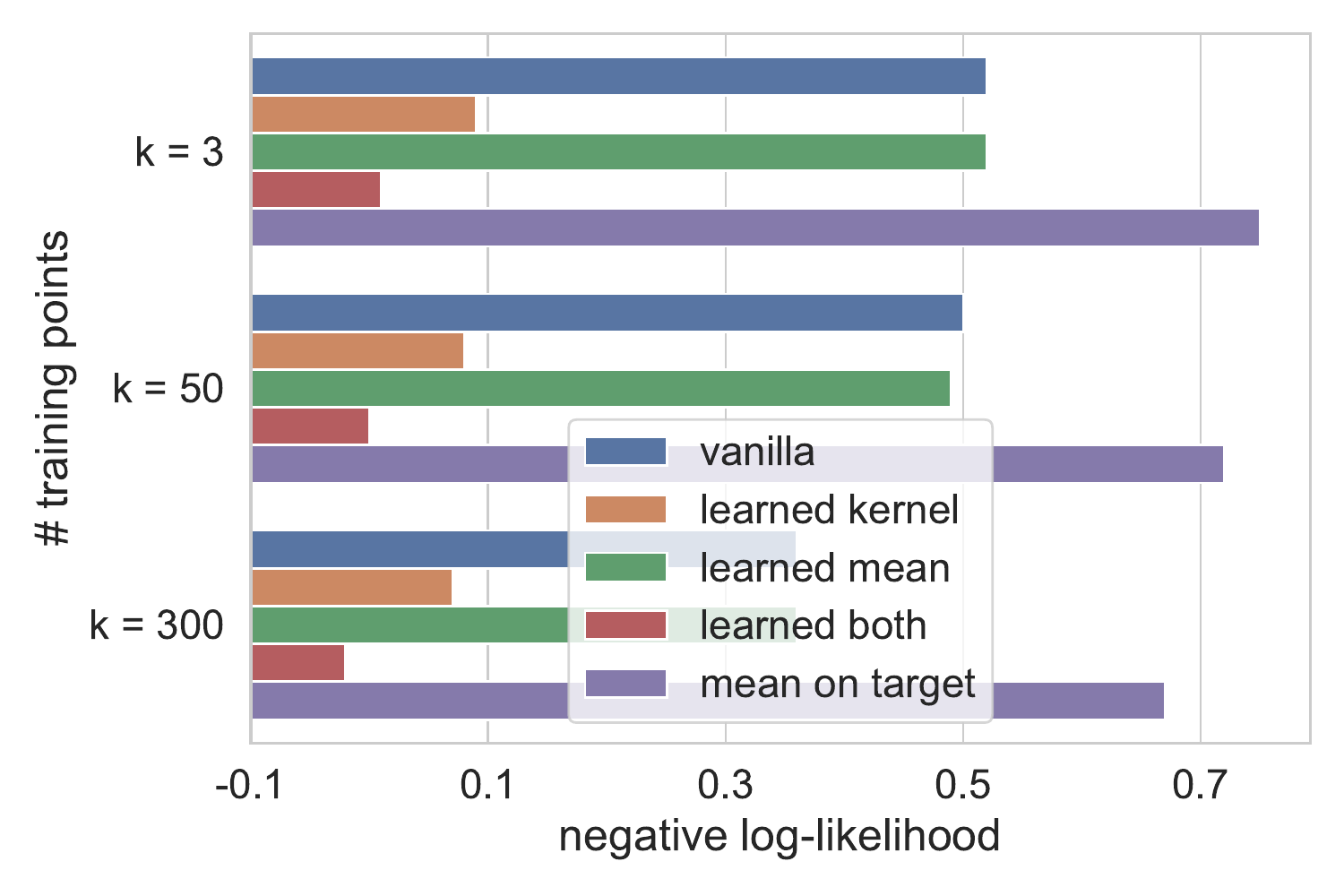}
	\end{subfigure}
	\caption{Performance comparison of the same methods as before on MNIST image completion for different numbers of training points. To illustrate overfitting in the non-meta-learning case, we also learned a mean function on the target training set directly (\emph{mean on target}). Moreover, we include a neural network without GP (\emph{standard NN}) as a baseline. The values are computed on the MNIST test set. The error bars are too small to be seen. Learning only the mean function outperforms the other methods in terms of MSE (also the neural network itself), while learning both mean and kernel function yields the best likelihoods. Moreover, learning the mean function directly on the target task leads to severe overfitting.}
	\label{fig:mnist_performance}
\end{figure*}

\paragraph{MNIST image completion}

\begin{figure}
\centering
\includegraphics[width=0.7\linewidth]{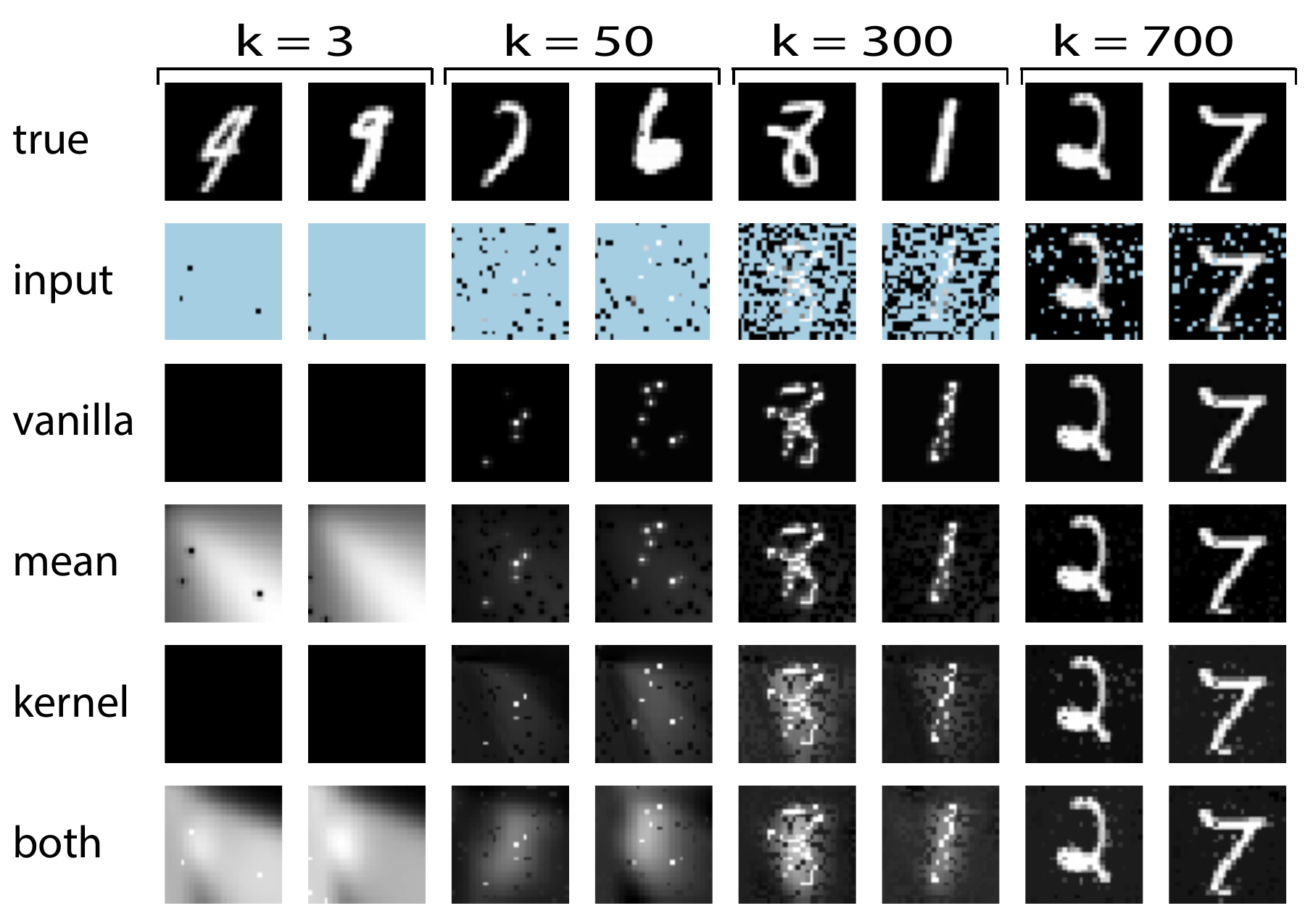}
\caption{Reconstructions of MNIST test digits using the different methods and different numbers of training points $k$.
}
\label{fig:mnist}
\end{figure}

To compare mean function learning to kernel learning in a more complex meta-learning setting, we performed an experiment with an image completion task on MNIST handwritten digits \citep{LeCun1998-es}.
We view the task of image completion, that is, inferring the pixel values in test locations given their values in some random context locations, as a regression task (similar to \citet{Garnelo2018-wb}).
We try to learn a function that maps from pixel coordinates to the pixel value, that is, $f : \{ 0, \dots, 27 \}^2 \rightarrow \left[ 0, 1 \right]$.
Some example inputs and reconstructions are depicted in Figure~\ref{fig:mnist} and quantitative results are shown in Figure~\ref{fig:mnist_performance} (and in the appendix in Tab.~\ref{tab:MNIST_completion}).
\newline
We observe that the deep kernel learning performs better than the vanilla GP and the mean function learning in terms of likelihood, but the mean function learning significantly outperforms the other methods in terms of MSE.
Encouragingly, the GP with the learned mean function also outperforms the neural network itself in terms of MSE (i.e., using the same network architecture without the GP).
\newline
Moreover, learning both mean function and kernel yields the best performance in terms of likelihood.
To illustrate the risk of overfitting mentioned in Section~\ref{sec:overfitting}, we also included a GP model in this experiment where we learned the mean function directly on the training set of the target task.
As expected, we observe that this model overfits the data and underperforms compared to all the other models (including the na\"ive baseline), highlighting the importance of the meta-learning setting for successful mean function learning.
\newline
These performance results fit the intuition of Section~\ref{sec:superior_mean}.
We can therefore confirm empirically that mean function learning can outperform kernel learning in terms of MSE under certain conditions and that it can also improve the performance in terms of likelihood when combined with kernel learning.

\begin{table*}
    \centering
    \caption{Performance comparison of the same methods as before on predicting patient trajectories on the intensive care unit for different clinical variables. The values are means and their standard errors of the 4,000 patients in the test set. The different variables are measured in different units and on different scales, thus one should pay more attention to relative differences between the methods rather than absolute ones. Learning the mean function alone or together with the kernel always outperforms learning the kernel alone.}
    \resizebox{\textwidth}{!}{
\begin{tabular}{lrrrrrr}
\toprule
& \multicolumn{3}{c}{Likelihood} & \multicolumn{3}{c}{MSE} \\
        Data & \multicolumn{1}{c}{learned mean} & \multicolumn{1}{c}{learned kernel} & \multicolumn{1}{c}{learned both} & \multicolumn{1}{c}{learned mean} & \multicolumn{1}{c}{learned kernel} & \multicolumn{1}{c}{learned both} \\
\midrule
          GCS &        -2.95 $\pm$ 0.02 &        -5.46 $\pm$ 0.03 &   \textbf{-2.44 $\pm$ 0.01} &        13.98 $\pm$ 0.24 &        82.93 $\pm$ 0.74 &         \textbf{9.15 $\pm$ 0.17} \\
        Urine &       \textbf{-25.14 $\pm$ 1.54} &       -40.20 $\pm$ 1.82 &  \textbf{-26.09 $\pm$ 1.61} &  \textbf{20674.24 $\pm$ 1505.23} &  35827.67 $\pm$ 1793.72 &  \textbf{21793.27 $\pm$ 1583.14} \\
          HCT &        \textbf{-2.93 $\pm$ 0.02} &       -18.17 $\pm$ 0.09 &   \textbf{-2.92 $\pm$ 0.01} &        \textbf{19.63 $\pm$ 0.50} &       837.31 $\pm$ 4.62 &        \textbf{19.76 $\pm$ 0.48} \\
   Creatinine &        -1.75 $\pm$ 0.05 &        -1.61 $\pm$ 0.02 &   \textbf{-1.52 $\pm$ 0.01} &         1.92 $\pm$ 0.16 &         1.58 $\pm$ 0.10 &         \textbf{0.91 $\pm$ 0.08} \\
          BUN &       -13.35 $\pm$ 0.58 &       -20.01 $\pm$ 0.62 &   \textbf{-8.88 $\pm$ 0.32} &      \textbf{441.54 $\pm$ 23.51} &     1058.40 $\pm$ 38.07 &      \textbf{412.11 $\pm$ 20.99} \\
\bottomrule
\end{tabular}
}    
\label{tab:ICU_prediction}
\end{table*}

\paragraph{Medical time series prediction}

To assess the performance of mean function learning on a challenging real-world task, we performed an experiment on medical time series data from the 2012 Physionet Challenge \citep{Silva2012predicting}.
These data correspond to the common meta-learning task of predicting the health trajectory of a patient who has recently been admitted to the hospital and for whom only few measurements are available yet.
The meta-tasks in this scenario are all the other patients that have been treated previously.
It stands to reason that the mean function could be even more important in this setting.
\newline
The results for some of the most frequently measured vital signs (for which the largest number of measurements are available) are reported in Table~\ref{tab:ICU_prediction}.
Note that the different variables are measured in different units and have different characteristic scales.
The models should therefore not be evaluated in terms of absolute performances, but rather be compared relative to each other.
We observe that the GPs that include a learned mean function (\emph{learned mean} and \emph{learned both}) consistently outperform the GP for which only the kernel has been learned.
This suggests that mean function learning can also lead to considerable benefits in challenging real-world tasks.

%% file: related_work.tex
\section{Related work}
\label{sec:related_work}

\paragraph{GP prior learning.}
Gaussian processes have gained a lot of attention in the field of machine learning in recent years \citep{Rasmussen2006-zv}, especially combined with approaches that improve their scalability \citep{Quinonero2005unifying, Titsias2009variational, Fortuin2018scalable, fortuin2019gp}.
Moreover, kernel learning for GPs has a long history \citep{Platt2002-pe}.
It can be seen as an extension to Maximum-Likelihood-II (ML-II) parameter optimization, where the kernel parameters are optimized with respect to the \emph{log marginal likelihood} of the GP \citep{Rasmussen2006-zv}.
The flexibility of these learned kernels grows with the number of parameters, culminating in deep kernel learning, where the kernel is parameterized by a deep neural network \citep{Wilson2015-ji, Wilson2016-dj}.
While kernels have thus enjoyed a lot of attention, the research into learning mean functions is still in its infancy.
Mean functions have been optimized over different levels of recursive autoregressive GPs in multifidelity modeling \citep{le2013bayesian, perdikaris2016multifidelity}, but these approaches have used very limited functional forms for the mean functions and have not considered meta-learning, but only transfer learning across different recursion levels on the same task.
The only documented deep mean function learning for GPs \citep{Iwata2017-vl} does not deal with a meta-learning setting, but with standard optimization of the mean function parameters on the training set, which can lead to severe overfitting (see Tab.~\ref{tab:MNIST_completion}).

\paragraph{Meta-learning.}
Meta-learning has recently been explored in great detail \citep{Vilalta2002-rj}, particularly in the case of deep neural networks \citep{Bengio2012-hw, Finn2017-ef}.
Especially the work on MAML \citep{Finn2017-ef} bears certain similarities to our proposed approach (see Sec.~\ref{sec:maml}).
It has been shown that these approaches can be seen as inference in a hierarchical Bayesian model \citep{Grant2018-dv}.
These explorations extend to GP-like neural network models \citep{Garnelo2018-wb, Garnelo2018-uy}, but the setting has been underexplored when it comes to classical GPs.
While there has been some work on meta-learning kernel functions \citep{Bonilla2008-rj, Widmer2010inferring, Skolidis2012-di}, parameters for Bayesian linear regression \citep{Harrison2018-zr}, and parameters in hierarchical Bayesian models \citep{Yu2005-ti}, the proposed GP mean function learning approach has not been studied in this setting yet.

%% file: discussion.tex
\section{Conclusion}
\label{sec:conclusion}

We have shown that meta-learning in Gaussian processes is an area that can benefit from learning the mean function of the process instead of (or combined with) the kernel function.
We have provided an analytical argument for why mean function meta-learning can be useful under certain conditions and why it poses less of a risk of overfitting compared to the standard supervised setting.
We have then validated this hypothesis empirically on benchmark tasks and real-world data.
This extends previous work in which mean function learning has not been considered for meta-learning purposes.
Moreover, we have drawn connections between mean function meta-learning and established meta-learning approaches such as MAML and functional PCA, thus lending further credence to our approach.

It would be an interesting direction for future work to explore more thoroughly under which conditions mean function learning is beneficial.
Moreover, combining mean function and kernel learning in even smarter ways could be a promising avenue of research.

%% file: supplement.tex
\beginsupplement

\clearpage

\onecolumn

\appendix

\section*{Appendix}

\FloatBarrier

\section{Counter-example against Assumption~\ref{lem:prior_mean}}
\label{sec:model_details}

Let us assume for simplicity that we generate data from a known noiseless process with nonzero mean, that is, $\exists x : m_{\text{true}}(x) \neq 0$, $\sigma^2 = 0$ and $\forall x : y(x) = m_{\text{true}}(x)$.
	
	If we want to fit a GP to these data and want it to yield correct predictions, we need the posterior to satisfy
	\begin{equation}
	\label{eq:proof_posterior}
	\begin{split}
		m_{true}(\mathbf{\tilde{x}^*}) = 
	&\; m_\phi(\mathbf{\tilde{x}^*}) + \\
	&\; K_\theta^{*x} \; K_\theta^{xx^{-1}} \, (m_{true}(\mathbf{\tilde{x}}) - m_\phi(\mathbf{\tilde{x}})) \; ,
	\end{split}
	\end{equation}
	where $\mathbf{\tilde{x}}$ and $\mathbf{\tilde{x}^*}$ are defined as above (Sec.~\ref{sec:meta-learning}).
	
	Let us now assume that we only happen to observe training points where the process is zero, but that it is nonzero for some test points, that is, $\forall x \in \mathbf{\tilde{x}} : m_{true}(x) = 0$ and $\exists x \in \mathbf{\tilde{x}^*} : m_{true}(x) \neq 0$ (this assumption is more likely to hold for small $\tilde{n}$).
	If we try to encode all our prior knowledge into the kernel parameters $\theta$ and choose an uninformative mean function with parameters $\phi$, that is, $m_\phi(\cdot) = 0$, the RHS of Equation~\ref{eq:proof_posterior} will always be zero, regardless of the choice of $\theta$, while the LHS will sometimes be nonzero.
	However, if we encode our prior knowledge into the mean function through $\phi$, that is, choose $m_\phi(\cdot) = m_{true}(\cdot)$, Equation~\ref{eq:proof_posterior} will always hold, regardless of the choice of kernel.
	
As mentioned above, the assumptions of this counter-example are more likely to hold for small values of $\tilde{n}$.
In fact, for $\tilde{n} \rightarrow \infty$, one can show that under mild assumptions we can always find a kernel that will make the posterior approach $m_{true}(\cdot)$ arbitrarily closely \citep{Micchelli2006-ul}.
However, the Bernstein-von Mises theorem tells us that asymptotically, any choice of GP prior is equally good anyway \citep{doob1949application}. 
In the case of scarce data (i.e., small $\tilde{n}$), which we assume to be more common in the meta-learning setting (see Sec.~\ref{sec:meta-learning}), we believe Proposition~\ref{lem:prior_mean} to be more relevant than the asymptotic kernel universality results.

\section{Proof of Proposition~\ref{prop:gp_fpca}}
\label{sec:fpca_proof}

We will show in the following that FPCA can be instantiated as a special case of our proposed framework.
For notational convenience, let us rewrite the GP posterior from \eqref{eq:gp-posterior} as
\begin{align*}
	p(\bystar \given \bxstar, \bx, \by) &= \gauss{\bmustar, \Sigma^*} \\
	\text{with} \quad \bmustar &= m(\bxstar) + \Kxstar^\top \left( \Kxx + \var \bI \right)^{-1} \left( \by - m(\bx) \right) \\
	\text{and} \quad \Sigma^* &= \var \bI + \Kstarstar - \Kxstar^\top \left( \Kxx + \var \bI \right)^{-1} \Kxstar \; .
\end{align*}

Let us now choose a specific form of mean function and kernel, namely $m(x) = \mathbf{b}(x)^\top \theta_m$ and $k(x, x') = \mathbf{b}(x)^\top \mathbf{b}(x')$, with some parameter $\theta_m$.

If we write $\Bstar$ for $\mathbf{B_{x^*}}$, the mean of the GP posterior then becomes
\begin{align*}
	\bmustar &= \Bstar \theta_m + \Bstar \Bx^\top \left( \Bx \Bx^\top + \var \bI \right)^{-1} (\by - \Bx \theta_m) \\
	&= \Bstar \theta_m + \frac{1}{\var} \Bstar \Bx^\top \left( \frac{1}{\var} \Bx \Bx^\top + \bI \right)^{-1} (\by - \Bx \theta_m) \\
	&= \Bstar \theta_m + \Bstar \left(\Bx^\top \Bx  + \var \bI \right)^{-1} \Bx^\top (\by - \Bx \theta_m) \;,
\end{align*}
where the last step uses one of the Searle identities \citep{searle2017matrix} which is equation~(162) in \citet{petersen2008matrix}.
If we now define a new parameter $\theta_c = \left(\Bx^\top \Bx  + \var \bI \right)^{-1} \Bx^\top (\by - \Bx \theta_m)$, the posterior mean becomes
\begin{equation}
\label{eq:fpca_mu}
	\bmustar = \Bstar (\theta_m + \theta_c) = \Bstar \hat{\beta} \; ,
\end{equation}
with $\hat{\beta} = \theta_m + \theta_c$.

Similarly, the posterior variance becomes
\begin{align*}
	\Sigma^* &= \var \bI + \Bstar \Bstar^\top - \Bstar \Bx^\top \left( \Bx \Bx^\top + \var \bI \right)^{-1} \Bx \Bstar^\top  \\
	&= \var \bI + \Bstar \left[ \bI - \Bx^\top \left( \Bx \Bx^\top + \var \bI \right)^{-1} \Bx \right] \Bstar^\top \\
	&= \var \bI + \Bstar \left[ \bI - \frac{1}{\var} \Bx^\top \left( \frac{1}{\var} \Bx \Bx^\top + \bI \right)^{-1} \Bx \right] \Bstar^\top \\
	&= \var \bI + \Bstar \left[ \bI + \frac{1}{\var} \Bx^\top  \Bx \right]^{-1} \Bstar^\top \; ,
\end{align*}
where the last step uses another Searle identity \citep{searle2017matrix}, namely equation~(166) in \citet{petersen2008matrix}.
We can then define $\hat{\Gamma} = \left( \bI + \frac{1}{\var} \Bx^\top  \Bx \right)^{-1}$, such that the posterior variance becomes
\begin{equation}
\label{eq:fpca_sigma}
	\Sigma^* = \var \bI + \Bstar \hat{\Gamma} \Bstar^\top \; .
\end{equation}

With \eqref{eq:fpca_mu} and \eqref{eq:fpca_sigma}, the whole GP posterior is then
\begin{equation}
	p(\bystar \given \bxstar ; \hat{\beta}, \hat{\Gamma}) = \gauss{\Bstar \hat{\beta}, \var \bI + \Bstar \hat{\Gamma} \Bstar^\top} \; ,
\end{equation}
which is exactly the distribution given by \eqref{eq:fpca} for the FPCA.

\section{Deep mean functions for GPs}
\label{sec:deep_mean}

Given the insight that meta-learning the GP's mean function might help us in solving our target task, we are still faced with the problem of choosing a suitable parameterization, that is, a suitable function class for the mean function.
Since we expect to have a reasonably large amount of data from the meta-tasks ($n_{\text{train}} \gg \tilde{n}$, see Sec.~\ref{sec:meta-learning}), we want to choose a function class that can scale easily to such amounts of data and that is flexible enough to incorporate the meta-knowledge.
A class of parametric functions that exhibit these two properties are deep neural networks \citep{LeCun2015-dw, Schmidhuber2015-ol, Goodfellow2016-fr}.

As an illustrative example, let us assume that we want to parameterize the mean function as a feed-forward neural network with two hidden layers.
The mean value at a point $\hat{x}$ will then be given by
\begin{equation}
	m_\phi(\hat{x}) = W_3 \sigma(W_2 \sigma(W_1 \hat{x} + b_1) + b_2) + b_3 \; ,
	\label{eq:deep_mean}
\end{equation}
where the $W_i$'s are the weight matrices of the layers, the $b_i$'s are their biases, and $\sigma(\cdot)$ is a nonlinear activation function (e.g., sigmoid or ReLU).
For our notation, we would then consume all the trainable parameters into the vector $\phi$, that is, $\phi = (W_i, b_i)_{i=1}^3$.

Given this functional form, we can train the mean function according to Section~\ref{sec:meta-learning}.
We compute the gradients of the LML with respect to $\phi$ using backpropagation, as implemented in various deep learning frameworks (in our experiments, we use \texttt{TensorFlow} \citep{Abadi2016tensorflow} and \texttt{PyTorch} \citep{Paszke2017automatic}).

It has been shown that using deep neural networks as kernel functions for GPs yields neural networks with nonparametric (or equivalently, ``infinitely wide'') final layers \citep{Wilson2015-ji}.
Similarly, using deep mean functions amounts to fitting a neural network to the data and then modeling the residuals of the fit with a GP \citep{Iwata2017-vl}.
It thus offers a natural way to combine the predictive power of neural networks with the calibrated uncertainties of GPs and can therefore be seen as a form of \emph{Bayesian Deep Learning} \citep{Gal2017-wz}.

\section{Experimental details}
\label{sec:exp_details}

This section contains additional details regarding the experiments (Sec.~\ref{sec:experiments}).
We implemented the algorithms in Python, using the \texttt{GPflow} package \citep{De_G_Matthews2017-ea} and the \texttt{GPyTorch} \citep{Gardner2018-bz} package.

\subsection{Sinusoid function regression}

As a proof of concept, we aim to assess the general performance of mean function meta-learning.
To this end, we simulated functions from a known generating Gaussian process.
The process had a sinusoid mean and a radial basis function (RBF) covariance.
For each function, we sampled the value at 50 equally spaced points in the $\left[ -5, 5 \right]$ interval.
Samples from the process are depicted in Figure~\ref{fig:sinusoid}.

We trained a deep feedforward neural network on 1000 sampled functions and used it as a mean function for GP regression on 200 unseen functions that were sampled from the same process.
We used a neural network with two hidden layers of size 64 each with sigmoid activation functions.
It was trained for 100 epochs using stochastic gradient descent (SGD).
We compared it against a GP with zero mean function and one with the true sinusoid mean function for different numbers of training points (Tab.~\ref{tab:sinusoid_regression}).

It can be seen that the learned mean function performs comparably with the true mean function and significantly better than the zero mean function in terms of likelihood and MSE.
It also becomes evident that this effect diminishes when the number of training points increases, following the intuition that the prior becomes less important once there are enough observations (see Sec.~\ref{sec:superior_mean}).
These findings support the efficacy of the mean function learning approach in the low-data limit.

\begin{figure}
	\centering
\includegraphics[width=0.6\linewidth]{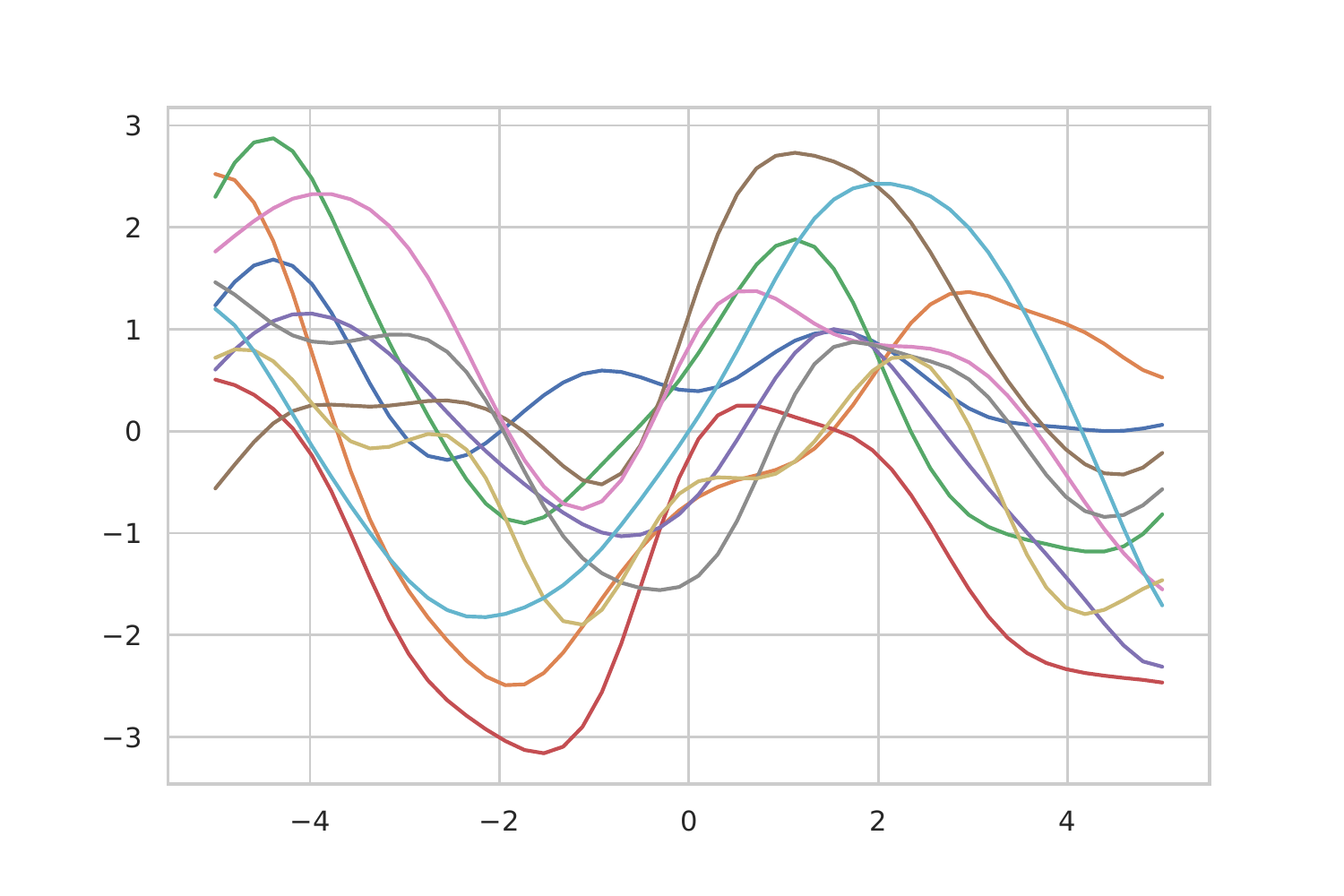}
\caption{Samples from the generative sinusoid process for the synthetic data experiment.}
\label{fig:sinusoid}
\end{figure}

\begin{table*}
    \centering
    \caption{Performance comparison of GP regression with a zero mean function, a learned mean function and the true mean function of the generating process on synthetic sinusoid function data for different numbers of training points. The performance is measured in terms of likelihood and mean squared error. The reported values are means and their respective standard errors of 200 runs. The learned mean function performs comparably with the true mean function.}
    \resizebox{\linewidth}{!}{
    \begin{tabular}{lrrrrrr}
        \toprule
         & \multicolumn{2}{c}{$\tilde{n} = 1$} & \multicolumn{2}{c}{$\tilde{n} = 5$} & \multicolumn{2}{c}{$\tilde{n} = 20$} \\
        Method & \multicolumn{1}{c}{likelihood} & \multicolumn{1}{c}{MSE} & \multicolumn{1}{c}{likelihood} & \multicolumn{1}{c}{MSE} & \multicolumn{1}{c}{likelihood} & \multicolumn{1}{c}{MSE} \\
         \midrule
         zero mean & -1.31 $\pm$ 0.07 & 1.21 $\pm$ 0.04 & -5.68 $\pm$ 0.13 & 1.02 $\pm$ 0.15 & 19.99 $\pm$ 0.22 & \textbf{0.00 $\pm$ 0.00} \\
         true mean & \textbf{-0.74 $\pm$ 0.08} & \textbf{0.79 $\pm$ 0.03} & \textbf{-4.32 $\pm$ 0.14} & \textbf{0.53 $\pm$ 0.04} & \textbf{21.55 $\pm$ 0.27} & \textbf{0.00 $\pm$ 0.00} \\
         learned mean & \textbf{-0.75 $\pm$ 0.08} & \textbf{0.80 $\pm$ 0.03} & \textbf{-4.43 $\pm$ 0.13} & \textbf{0.52 $\pm$ 0.03} & 20.76 $\pm$ 0.22 & \textbf{0.00 $\pm$ 0.00} \\
         \bottomrule
    \end{tabular}
    }
    \label{tab:sinusoid_regression}
\end{table*}

\subsection{Step function regression}

A step function is simply defined as
\begin{equation}
	f_{\text{step}} (x) =
	\begin{cases}
		y_1 & \text{if} \quad x < x_{\text{step}} \\
		y_2 & \text{otherwise}
	\end{cases}
\end{equation}
with the step location $x_{\text{step}}$ and the two function values $(y_1, y_2)$ before and after the step, respectively.
For our set of tasks, we choose a dataset of different step functions, namely the Heaviside step function \citep{Bracewell1986fourier} (i.e., $(y_1, y_2) = (0, 1)$) and its mirrored version along the $x_{\text{step}}$-axis (i.e., $(y_1, y_2) = (1, 0)$).
We sample the step location $x_{\text{step}}$ uniformly at random from the $\left[ -1, 1 \right]$ interval.
The whole function consists of 50 evenly spaced points in the $\left[ -2, 2 \right]$ domain.

We compare a vanilla GP (zero mean and RBF kernel) with a learned deep mean function, a learned deep kernel function (following \citet{Wilson2015-ji}) and a GP with deep mean and deep kernel function learned concurrently.
The learned mean and kernel functions are parameterized by the same deep feed-forward neural network architecture (except for the dimension of the last layer).
We use neural networks with two hidden layers of size 128 and 64 with sigmoid activation functions.
We train each model on 10,000 randomly sampled step functions from our function space using stochastic gradient descent.

\subsection{MNIST image completion}

\begin{table*}[h]
    \centering
    \caption{Performance comparison of the same methods as before on MNIST image completion for different numbers of training points. To illustrate overfitting in the non-meta-learning case, we also learned a mean function on the target training set directly (\emph{mean fit on target}). Moreover, we include a neural network without GP (\emph{neural network}) as a baseline. The values are means and their standard errors of the 10,000 images in the MNIST test set. Learning only the mean function outperforms the other methods in terms of MSE (also the neural network itself), while learning both mean and kernel function yields the best likelihoods. Moreover, learning the mean function directly on the target task leads to severe overfitting.}
    \resizebox{\textwidth}{!}{
    \begin{tabular}{lrrrrrr}
\toprule
& \multicolumn{2}{c}{$\tilde{n} = 3$} & \multicolumn{2}{c}{$\tilde{n} = 50$} & \multicolumn{2}{c}{$\tilde{n} = 300$} \\
        Method & \multicolumn{1}{c}{likelihood} & \multicolumn{1}{c}{MSE} & \multicolumn{1}{c}{likelihood} & \multicolumn{1}{c}{MSE} & \multicolumn{1}{c}{likelihood} & \multicolumn{1}{c}{MSE} \\
\midrule
neural network & - & 0.097 $\pm$ 0.000 & - & 0.097 $\pm$ 0.000 & - & 0.097 $\pm$ 0.000 \\
mean fit on target & -0.75 $\pm$ 0.00 & 0.200 $\pm$ 0.006 & -0.72 $\pm$ 0.00 & 0.082 $\pm$ 0.001 & -0.67 $\pm$ 0.00 & 0.041 $\pm$ 0.000 \\
\midrule
vanilla \citep{Rasmussen2006-zv}   &  -0.52 $\pm$ 0.00 &  0.112 $\pm$ 0.000 &  -0.50 $\pm$ 0.00 &  0.082 $\pm$ 0.000 &  -0.36 $\pm$ 0.00 &  0.022 $\pm$ 0.000 \\
learned kernel \citep{Wilson2015-ji}       &  -0.09 $\pm$ 0.00 &  0.166 $\pm$ 0.001 &  -0.08 $\pm$ 0.00 &  0.075 $\pm$ 0.000 &  -0.07 $\pm$ 0.00 &  0.067 $\pm$ 0.000 \\
learned mean (ours)       &  -0.52 $\pm$ 0.00 &  \textbf{0.090 $\pm$ 0.000} &  -0.49 $\pm$ 0.00 &  \textbf{0.068 $\pm$ 0.000} &  -0.36 $\pm$ 0.00 &  \textbf{0.021 $\pm$ 0.000} \\
learned both (ours)       &  \textbf{-0.01 $\pm$ 0.00} &  0.151 $\pm$ 0.001 &   \textbf{0.00 $\pm$ 0.00} &  0.073 $\pm$ 0.000 &   \textbf{0.02 $\pm$ 0.00} &  0.059 $\pm$ 0.000 \\
\bottomrule
\end{tabular}
}    
\label{tab:MNIST_completion}
\end{table*}

The learned mean and kernel functions are parameterized by the same deep feed-forward neural network architecture (except for the dimension of the last layer).
We use neural networks with two hidden layers of size 128 and 64 with sigmoid activation functions.
We train each model on the MNIST training set of images using stochastic gradient descent.
We again compare the two approaches against a ``vanilla'' GP (zero mean and RBF kernel) and a GP where both functions are learned concurrently.
We report performances on the MNIST test set with different numbers of training points (Tab.~\ref{tab:MNIST_completion}).

\subsection{Medical time series prediction}

In this data set, there are 4,000 patients in the training set and 4,000 in the test set.
The patients are observed for 48 hours on the intensive care unit (ICU), where different vital signs are measured at different frequencies.
During meta-training, we learn the mean function and kernel function of the GP on the whole time series of 48 hours.
We use the same neural network architectures as in the MNIST experiment.
We learn a different GP prior for each of the different vital signs.
During testing, we condition a GP with the meta-learned prior on the first 24 hours and try to predict the remaining 24 hours.